\documentclass[letterpaper]{article}
\usepackage[pass]{geometry}
\usepackage{graphicx} % more modern
\usepackage{epsfig} % less modern
\usepackage{subfigure,ifsym} 
\usepackage{times}
\usepackage{natbib}

\usepackage{algorithm}
\usepackage{algorithmic}

\usepackage[accepted]{icml2014} 
\usepackage{amsmath,amssymb,natbib,amsthm}

\usepackage{url} 
\usepackage{tabularx,color} 
\usepackage{dsfont} 
\usepackage{epsf} 
\usepackage{algorithmic,algorithm,mathtools} 
\usepackage{wrapfig}

\usepackage{graphicx}

\newtheorem{definition}{Definition}

\newtheorem{corollary}{Corollary}
\newtheorem{lemma}{Lemma}
\newtheorem{proposition}{Proposition}

\def\bp{\mathbf{p}}
\def\bd{\mathbf{d}}
\DeclareMathOperator{\diag}{diag}

\DeclareMathOperator{\tr}{tr}

\DeclareMathOperator{\argmin}{argmin}
\DeclareMathOperator{\defeq}{\overset{\defi}{=}}
\providecommand{\norm}[1]{\lVert#1\rVert} 
\def\bd{\mathbf{d}}

\def\bx{\mathbf{x}}
\def\by{\mathbf{y}}

\def\defeq{\overset{\textup{\text{def}}}{=}}

\def\one{\mathds{1}} 
\newcommand{\dotprod}[2]{\ensuremath{\langle #1 , #2\,\rangle}}

\icmltitlerunning{Fast Computation of Wasserstein Barycenters}

\begin{document} 

\twocolumn[
\icmltitle{Fast Computation of Wasserstein Barycenters}

\icmlauthor{Marco Cuturi}{mcuturi@i.kyoto-u.ac.jp}
\icmladdress{Graduate School of Informatics,
Kyoto University}
\icmlauthor{Arnaud Doucet}{doucet@stat.oxford.ac.uk}
\icmladdress{Department of Statistics,
            University of Oxford}

\icmlkeywords{optimal transportation, Wasserstein barycenter, Matrix scaling, Sinkhorn-Knopp}

\vskip 0.3in
]
\sloppy

\begin{abstract}We present new algorithms to compute the mean of a set of empirical probability measures under the optimal transport metric. This mean, known as the Wasserstein barycenter, is the measure that minimizes the sum of its Wasserstein distances to each element in that set. We propose two original algorithms to compute Wasserstein barycenters that build upon the subgradient method. A direct implementation of these algorithms is, however, too costly because it would require the repeated resolution of large primal and dual optimal transport problems to compute subgradients. Extending the work of \citet{cuturi2013sinkhorn}, we propose to smooth the Wasserstein distance used in the definition of Wasserstein barycenters with an entropic regularizer and recover in doing so a strictly convex objective whose gradients can be computed for a considerably cheaper computational cost using matrix scaling algorithms. We use these algorithms to visualize a large family of images and to solve a constrained clustering problem.
\end{abstract}

%----------------------------------------
\section{Introduction}
%----------------------------------------

\begin{figure}[h!]
\centering
\includegraphics[width=7.7cm]{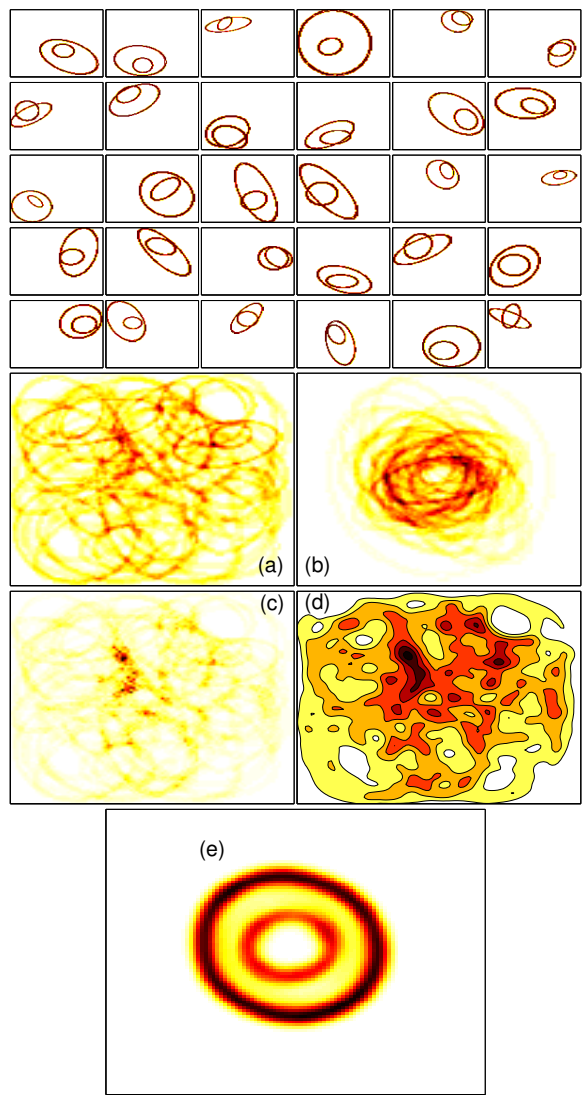}
\caption{(Top) 30 artificial images of two nested random ellipses. Mean measures using the (a) Euclidean distance (b) Euclidean after re-centering images (c) Jeffrey centroid \citep{JeffreysCentroid-2013} (d) RKHS distance (Gaussian kernel, $\sigma=0.002$) (e) 2-Wasserstein distance.}\label{fig:nines}\vskip-.5cm
\end{figure}

Comparing, summarizing and reducing the dimensionality of empirical probability measures defined on a space $\Omega$ are fundamental tasks in statistics and machine learning. Such tasks are usually carried out using pairwise comparisons of measures. Classic information divergences \citep{amar01b} are widely used to carry out such comparisons. 

Unless $\Omega$ is finite, these divergences cannot be directly applied to empirical measures, because they are ill-defined for measures that do not have continuous densities. They also fail to incorporate prior knowledge on the geometry of $\Omega$, which might be available if, for instance, $\Omega$ is also a Hilbert space. Both of these issues are usually solved using \citeauthor{parzen1962estimation}'s approach \citeyearpar{parzen1962estimation} to smooth empirical measures with smoothing kernels before computing divergences: the Euclidean \citep{gretton06} and $\chi_2$ distances \citep{harchaoui07}, the Kullback-Leibler and Pearson divergences \citep{kanamori2012,kanamori2012statistical} can all be computed fairly efficiently by considering matrices of kernel evaluations.

The choice of a divergence defines implicitly the \emph{mean} element, or barycenter, of a set of measures, as the particular measure that minimizes the sum of all its divergences to that set of target measures \citep{veldhuis2002centroid,banerjee2005clustering,teboulle2007unified,JeffreysCentroid-2013}. The goal of this paper is to compute efficiently barycenters (possibly in a constrained subset of all probability measures on $\Omega$) defined by the \emph{optimal transport distance} between measures \citep[\S6]{villani09}. We propose to minimize directly the sum of optimal transport distances from one measure (the variable) to a set of fixed measures by gradient descent. These gradients can be computed for a moderate cost by solving smoothed optimal transport problems as proposed by \citet{cuturi2013sinkhorn}.

% 
% . In addition to producing well-behaved densities from empirical measures, a smoothing kernel has two virtues: it encodes prior knowledge on the probability space $\Omega$, and greatly facilitates computations if positive definite. For instance, the Euclidean \citep{gretton06} and $\chi_2$ distances \citep{harchaoui07}, the Kullback-Leibler and Pearson divergences \citep{kanamori2012,kanamori2012statistical} can all be computed fairly efficiently by considering matrices of kernel evaluations.

Wasserstein distances have many favorable properties, documented both in theory \citep{villani09} and practice \citep{rubner1997earth,Pele-iccv2009}. We argue that their versatility extends to the barycenters they define.
We illustrate this intuition in Figure~\ref{fig:nines}, where we consider 30 images of nested ellipses on a $100\times 100$ grid. Each image is a discrete measure on $[0,1]^2$ with normalized intensities. Computing the Euclidean, Gaussian RKHS mean-maps or Jeffrey centroid of these images results in mean measures that hardly make any sense, whereas the 2-Wasserstein mean on that grid (defined in \S\ref{subsec:defbaryc}) produced by Algorithm~\ref{algo:discwass} captures perfectly the structure of these images. Note that these results were recovered without any prior knowledge on these images other than that of defining a distance in $[0,1]^2$, here the Euclidean distance. Note also that the Gaussian kernel smoothing approach uses the same distance, in addition to a bandwidth parameter $\sigma$ which needs to be tuned in practice.
%We mention in Section~\ref{subsec:relevance} other interesting connections and applications of Wasserstein barycenters.

This paper is organized as follows: we provide background on optimal transport in~\S\ref{sec:back}, followed by the definition of Wasserstein barycenters with motivating examples in~\S\ref{sec:baryc}. Novel contributions are presented from~\S\ref{sec:computing}: we present two subgradient methods to compute Wasserstein barycenters, one which applies when the support of the mean measure is known in advance and another when that support can be freely chosen in $\Omega$. These algorithms are very costly even for measures of small support or histograms of small size. We show in~\S\ref{sec:smooth} that the key ingredients of these approaches---the computation of primal and dual optimal transport solutions---can be bypassed by solving smoothed optimal transport problems. We conclude with two applications of our algorithms in \S\ref{sec:exp}.

%----------------------------------------
\section{Background on Optimal Transport}\label{sec:back}
%----------------------------------------
Let $\Omega$ be an arbitrary space, $D$ a metric on that space and $P(\Omega)$ the set of Borel probability measures on $\Omega$. For any point $x\in\Omega$, $\delta_{x}$ is the Dirac unit mass on $x$.
\begin{definition}[Wasserstein Distances]
For $p\in[1,\infty)$ and probability measures $\mu,\nu$ in $P(\Omega)$, their $p$-Wasserstein distance \citep[\S6]{villani09} is
$$%\begin{equation}\label{eq:wass}
	W_p(\mu,\nu)\mathrm{\defeq}\left(\inf_{\pi\in\Pi(\mu,\nu)}\int_{\Omega^2} D(x,y)^p d\pi(x,y) \right)^{1/p},
$$%%\end{equation}
where $\Pi(\mu,\nu)$ is the set of all probability measures on $\Omega^2$ that have marginals $\mu$ and $\nu$.
\end{definition}
%%%%%%%%%%%%%%%%%%%%%%%%%%%%%%%%%%%%%%%%%%%%%%%%%%%%%%%%
\subsection{Restriction to Empirical Measures}\label{subsec:finite}
%%%%%%%%%%%%%%%%%%%%%%%%%%%%%%%%%%%%%%%%%%%%%%%%%%%%%%%%
We will only consider empirical measures throughout this paper, that is measures of the form $\mu=\sum_{i=1}^n a_i \delta_{x_i}$ where $n$ is an integer,  $X=(x_1,\dots,x_n)\in\Omega^n$ and $(a_1,\dots,a_n)$ lives in the probability simplex $\Sigma_n$,
 $$\Sigma_n\defeq \{u\in\mathbb{R}^n\;|\;\forall i\leq n , u_i\geq 0, \sum_{i=1}^n u_i=1\}.$$
Let us introduce additional notations:

\textbf{Measures on a Set $X$ with Constrained Weights.} 
Let $\Theta$ be a non-empty closed subset $\Theta$ of $\Sigma_n$. We write
$$
P(X,\Theta)\defeq\{\mu=\sum_{i=1}^n a_i \delta_{x_i},\; a\in\Theta\}.
$$%\end{equation}
\textbf{Measures supported on up to $k$ points.} Given an integer $k$ and a subset $\Theta$ of $\Sigma_k$, we consider the set $P_k(\Omega,\Theta)$ of measures of $\Omega$ that have discrete support of size up to $k$ and weights in $\Theta$,
$$P_k(\Omega,\Theta)\defeq \bigcup_{X\in\Omega^k}P(X,\Theta).$$
When no constraints on the weights are considered, namely when the weights are free to be chosen anywhere on the probability simplex, we use the shorter notations $P(X)\defeq P(X,\Sigma_n)$ and $P_k(\Omega)\defeq P_k(\Omega,\Sigma_k)$. %Note that the set of all empirical measures of $\Omega$ is the union of all sets $P_k(\Omega), k\geq 1$.

%%%%%%%%%%%%%%%%%%%%%%%%%%%%%%%%%%%%%%%%%%%%%%%%%%%%%%%%
\subsection{Wasserstein \& Discrete Optimal Transport}\label{subsec:wdandot} 
%%%%%%%%%%%%%%%%%%%%%%%%%%%%%%%%%%%%%%%%%%%%%%%%%%%%%%%%
Consider two families $X=(x_1,\dots,x_n)$ and $Y=(y_1,\dots,y_m)$ of points in $\Omega$. When $\mu=\sum_{i=1}^n a_i \delta_{x_i}$ and $\nu=\sum_{i=1}^n b_i \delta_{y_i}$, the Wasserstein distance $W_p(\mu,\nu)$ between $\mu$ and $\nu$ is the $p^\text{th}$ root of the optimum of a network flow problem known as the \emph{transportation problem} \citep[\S7.2]{bertsimas1997introduction}. This problem builds upon two elements: the \emph{\textbf{matrix}} $M_{XY}$ \emph{\textbf{of pairwise distances}} between elements of $X$ and $Y$ raised to the power $p$, which acts as a cost parameter,
\begin{equation}\label{eq:distMatrix}M_{XY}\defeq [D(x_i,y_j)^p]_{ij} \in\mathbb{R}^{n\times m},\end{equation} 
	and the \emph{\textbf{transportation polytope}} $U(a,b)$ of $a\in\Sigma_n$ and $b\in\Sigma_m$, which acts as a feasible set, defined as the set of $n\times m$ nonnegative matrices such that their row and column marginals are equal to $a$ and $b$ respectively. Writing $\one_n$ for the $n$-dimensional vector of ones,
\begin{equation}\label{eq:polytope} 
	U(a,b) \defeq \{ T\in\mathbb{R}_+^{n\times m}\; |\; T\one_m = a,\, T^T\one_n = b \}.
\end{equation}
	
Let $\dotprod{A}{B}\defeq \tr(A^T B)$ be the Frobenius dot-product of matrices.  Combining Eq.~\eqref{eq:distMatrix} \& \eqref{eq:polytope}, we have that $W_p^p(\mu,\nu)$---the distance $W_p(\mu,\nu)$ raised to the power $p$---can be written as the optimum of a parametric linear program $\bp$ on $n\times m$ variables, parameterized by the marginals $a,b$ and a (cost) matrix $M_{XY}$:
\begin{multline}\label{eq:primal}W_p^p(\mu,\nu) = \bp(a,b,M_{XY})\defeq\min_{T\in U(a,b)}\dotprod{T}{M_{XY}}.\end{multline}

%%%%%%%%%%%%%%%%%%%%%%%%%%%%%%%%%%%%%%%%%%%%%%%%%%%%%%%%
\section{Wasserstein Barycenters}\label{sec:baryc}
%%%%%%%%%%%%%%%%%%%%%%%%%%%%%%%%%%%%%%%%%%%%%%%%%%%%%%%%
We present in this section the Wasserstein barycenter problem, a variational problem involving all Wasserstein distances from one to many measures, and show how it encompasses known problems in clustering and approximation.
\subsection{Definition and Special Cases}\label{subsec:defbaryc}
\begin{definition}[\citeauthor{agueh2011barycenters}, \citeyear{agueh2011barycenters}] A Wasserstein barycenter of $N$ measures $\{\nu_1,\dots,\nu_N\}$ in $\mathbb{P}\subset P(\Omega)$ is a minimizer of $f$ over $\mathbb{P}$, where
\begin{equation}\label{eq:f}f(\mu)\defeq \frac{1}{N}\sum_{i=1}^N W_p^p(\mu,\nu_i).\end{equation}
\end{definition}
%The variational formulation\footnote{\citeauthor{agueh2011barycenters} consider more generally non-uniform weights on the distances from $\mu$ to the $N$ %target measures. The algorithms we propose extend trivially to that case. We use uniform weights in this work to keep notations simpler.} of barycenters used by %\citeauthor{agueh2011barycenters} is similar to that used by \citet{veldhuis2002centroid} or \citet{banerjee2005clustering}. 
%This idea can be traced back to the concept of \citeauthor{frechet1948elements} means \citeyearpar{frechet1948elements}, as pointed out by %\citet{bigot2012consistent}.

\citeauthor{agueh2011barycenters} consider more generally a non-negative weight $\lambda_i$ in front of each distance $W_p^p(\mu,\nu_i)$. The algorithms we propose extend trivially to that case but we use uniform weights in this work to keep notations simpler.

% of barycenters used by %\citeauthor{agueh2011barycenters} is similar to that used by \citet{veldhuis2002centroid} or \citet{banerjee2005clustering}. 
%This idea can be traced back to the concept of \citeauthor{frechet1948elements} means \citeyearpar{frechet1948elements}, as pointed out by %\citet{bigot2012consistent}.

We highlight a few special cases where minimizing $f$ over a set $\mathbb{P}$ is either trivial, relevant to data analysis and/or has been considered in the literature with different tools or under a different name. In what follows $X\in\Omega^n$ and $Y\in\Omega^m$ are arbitrary finite subsets of $\Omega$.

\textbf{\textbullet\; $N=1, \mathbb{P}=P(X)$} When only one measure $\nu$, supported on $Y\in\Omega^m$ is considered, its closest element $\mu$ in $P(X)$---if no constraints on weights $a$ are given---can be computed by defining a weight vector $a$ on the elements of $X$ that results from assigning all of the mass $b_i$ to the closest neighbor in metric $D$ of $y_i$ in $X$.\\

\textbf{\textbullet\; Centroids of Histograms}: $N>1, \Omega$ finite, $\mathbb{P}=P(\Omega)$. When $\Omega$ is a set of size $d$ and a matrix $M\in\mathbb{R}^{d\times d}_+$ describes the pairwise distances between these $d$ points (usually called in that case bins or features), the $1$-Wasserstein distance is known as the Earth Mover's Distance (EMD) \citep{rubner1997earth}. In that context, Wasserstein barycenters have also been called EMD prototypes by \citet{zen2011earth}.
%Suppose, for instance, that $\Omega$ is the set of all words in all languages. Suppose one is given a set of documents supported on a language $Y\subset\Omega$. Which bag-of-words, supported in a different language $X$, could summarize most efficiently such a set? Such problems appear naturally in cross-lingual document retrieval \citep{kraaij2003embedding} and cross-lingual document categorization \citep{nastasebridging}.\\

\textbf{\textbullet\; Euclidean  $\Omega$}: $N=1$, $D(x,y)=\norm{x-y}_2, p=2, \mathbb{P} = P_k(\Omega)$.
Minimizing $f$ on $P_k(\Omega)$ when $(\Omega,D)$ is a Euclidean metric space and $p=2$ is equivalent to the $k$-means problem \citep{pollard1982quantization,canas2012learning}.

\textbf{\textbullet\ Constrained $k$-Means}:  $N=1$,  $\mathbb{P} = P_k(\Omega,\{\one_k/k\})$. 
Consider a measure $\nu$ with support $Y\in\Omega^m$ and weights $b\in\Sigma_m$. 
The problem of approximating this measure by a uniform measure with $k$ atoms---a measure in $P_k(\Omega,\{\one_k/k\})$---in $2$-Wasserstein sense was to our knowledge first considered by \citet{ng2000note}, who proposed a variant of \citeauthor{lloyd1982least}'s algorithm \citeyearpar{lloyd1982least} for that purpose. More recently, \citet{reich2012} remarked that such an approximation can be used in the resampling step of particle filters and proposed in that context two ensemble methods inspired by optimal transport, one of which reduces to a single iteration of \citeauthor{ng2000note}'s algorithm. Such approximations can also be obtained with kernel-based approaches, by minimizing an information divergence between the (smoothed) target measure $\nu$ and its (smoothed) uniform approximation as proposed recently by \citet{CheWelSmo10} and \citet{sugiyama2011information}.

% Kernel based approaches that minimize an information divergence (computed between smoothed empirical measures) between an arbitrary empirical measure and its uniform approximation have been proposed by \citet{CheWelSmo10} and \citet{sugiyama2011information} using Gaussian kernels.

\subsection{Recent Work}\label{subsec:recentwork}
\citet{agueh2011barycenters} consider conditions on the $\nu_i$'s for a Wasserstein barycenter in $P(\Omega)$ to be unique using the multi-marginal transportation problem. They provide solutions in the cases where either (i) $\Omega=\mathbb{R}$; (ii) $N=2$ using \citeauthor{mccann1997convexity}'s interpolant \citeyearpar{mccann1997convexity}; (iii) all the measures $\nu_i$ are Gaussians in $\Omega=\mathbb{R}^d$, in which case the barycenter is a Gaussian with the mean of all means and a variance matrix which is the unique positive definite root of a matrix equation \citep[Eq.6.2]{agueh2011barycenters}.

\citet{rabin2012} were to our knowledge the first to consider practical approaches to compute Wasserstein barycenters between point clouds in $\mathbb{R}^d$. To do so, \citet{rabin2012} propose to approximate the Wasserstein distance between two point clouds by their \emph{sliced} Wasserstein distance, the expectation of the Wasserstein distance between the projections of these point clouds on lines sampled randomly. Because the optimal transport between two point clouds on the real line can be solved with a simple sort, the sliced Wasserstein barycenter can be computed very efficiently, using  gradient descent. Although their approach seems very effective in lower dimensions, it may not work for $d\geq 4$ and does not generalize to non-Euclidean metric spaces.

%%%%%%%%%%%%%%%%%%%%%%%%%%%%%%%%%%%%%%%%%%%%%%%%%%%%%%%%
\section{New Computational Approaches}\label{sec:computing}
%%%%%%%%%%%%%%%%%%%%%%%%%%%%%%%%%%%%%%%%%%%%%%%%%%%%%%%%
We propose in this section new approaches to compute Wasserstein barycenters when (i) each of the $N$ measures $\nu_i$ is an empirical measure, described by a list of atoms $Y_i\in\Omega^{m_i}$ of size $m_i\geq 1$, and a probability vector $b_i$ in the simplex $\Sigma_{m_i}$;  (ii) the search for a barycenter is not considered on the whole of $P(\Omega)$ but restricted to either $P(X,\Theta)$ (the set of measures supported on a predefined finite set $X$ of size $n$ with weights in a subset $\Theta$ of $\Sigma_{n}$) or $P_k(\Omega,\Theta)$ (the set of measures supported on up to $k$ atoms with weights in a subset $\Theta$ of $\Sigma_k$). 

Looking for a barycenter $\mu$ with atoms $X$ and weights $a$ is equivalent to minimizing $f$ (see Eq.~\ref{eq:primal} for a definition of $\bp$),
\begin{equation}\label{eq:f}f(a,X) \defeq \frac{1}{N} \sum_{i=1}^N \bp(a,b_i,M_{XY_i}),\end{equation}
over relevant feasible sets for $a$ and $X$. When $X$ is \emph{fixed}, we show in \S\ref{subsec:dualityconvexity} that $f$ is convex w.r.t $a$ regardless of the properties of $\Omega$. A subgradient for $f$ w.r.t $a$ can be recovered through the \emph{dual optimal solutions} of all problems $\bp(a,b_i,M_{XY_i})$, and $f$ can be minimized using a projected subgradient method outlined in \S\ref{subsec:Xrestricted}. If $X$ is \emph{free}, constrained to be of cardinal $k$, and $\Omega$ and its metric $D$ are both \emph{Euclidean}, we show in \S\ref{subsec:XEuclidean} that $f$ is not convex w.r.t $X$ but we can provide subgradients for $f$ using the \emph{primal optimal solutions} of all problems $\bp(a,b_i,M_{XY_i})$. This in turn suggests an algorithm to reach a local minimum for $f$ w.r.t. $a$ and $X$ in $P_k(\Omega,\Theta)$ by combining both approaches.

%%%%%%%%%%%%%%%%%%%%%%%%%%%%%%%%%%%%%%%%%%%%%%%%%%%%%%%%
\subsection{Differentiability of $\bp(a,b,M_{XY})$ w.r.t $a$}\label{subsec:dualityconvexity}
%%%%%%%%%%%%%%%%%%%%%%%%%%%%%%%%%%%%%%%%%%%%%%%%%%%%%%%%
\textbf{Dual transportation problem.} Given a matrix $M\in\mathbb{R}^{n\times m}$, the optimum $\bp(a,b,M)$ admits the following dual Linear Program (LP) form \citep[\S7.6,\S7.8]{bertsimas1997introduction}, known as the dual optimal transport problem:
\begin{equation}\label{eq:dual}\bd(a,b,M)=\max_{(\alpha,\beta)\in C_M} \alpha^Ta+\beta^Tb\,,\end{equation} 
where the polyhedron $C_M$ of dual variables is 
$$C_M=\{(\alpha,\beta)\in\mathbb{R}^{n+m} \,|\, \alpha_i+\beta_j\leq m_{ij}\}.$$
By LP duality, $\bd(a,b,M)=\bp(a,b,M)$. The dual optimal solutions---which can be easily recovered from the primal optimal solution \citep[Eq.7.10]{bertsimas1997introduction}---define a subgradient for $\bp$ as a function of $a$:
\begin{proposition}\label{prop:convex} Given $b\in\Sigma_m$ and $M\in\mathbb{R}^{n\times m}$, the map $a\mapsto \bp(a,b,M)$ is a polyhedral convex function. Any optimal dual vector $\alpha^\star$ of $\bd(a,b,M)$ is a subgradient of $\bp(a,b,M)$ with respect to $a$.
\end{proposition}
\begin{proof} These results follow from sensitivity analysis in LP's \citep[\S5.2]{bertsimas1997introduction}. $\bd$ is bounded and is also the maximum of a finite set of linear functions, each indexed by the set of extreme points of $C_M$, evaluated at $a$ and is therefore polyhedral convex. When the dual optimal vector is unique, $\alpha^\star$ is a gradient of $\bp$ at $a$, and a subgradient otherwise.\end{proof}
Because for any real value $t$ the pair $(\alpha+t\one_n,\beta-t\one_m)$ is feasible if the pair $(\alpha,\beta)$ is feasible, and because their objective are identical, any dual optimum $(\alpha,\beta)$ is determined up to an additive constant. To remove this degree of freedom---which arises from the fact that one among all $n+m$ row/column sum constraints of $U(a,b)$ is redundant---we can either remove a dual variable or normalize any dual optimum $\alpha^\star$ so that it sums to zero, to enforce that it belongs to the tangent space of $\Sigma_n$. We follow the latter strategy in the rest of the paper.

%%%%%%%%%%%%%%%%%%%%%%%%%%%%%%%%%%%%%%%%%%%%%%%%%%%%%%%%
\subsection{Fixed Support: Minimizing $f$ over $P(X)$}\label{subsec:Xrestricted}
%%%%%%%%%%%%%%%%%%%%%%%%%%%%%%%%%%%%%%%%%%%%%%%%%%%%%%%%
Let $X\subset\Omega^n$ be fixed and let $\Theta$ be a closed convex subset of $\Sigma_n$. The aim of this section is to compute weights $a\in\Theta$ such that $f(a,X)$ is minimal. Let $\alpha_i^\star$ be the optimal dual variable of $\bd(a,b_i;M_{XY_i})$ normalized to sum to 0. $f$ being a sum of terms $\bp(a,b_i,M_{XY_i})$, we have that:

\begin{corollary} The function  $a\mapsto f(a,X)$ is polyhedral convex, with subgradient $$\pmb{\alpha}\defeq \frac{1}{N}\sum_{i=1}^N \alpha_i^\star.$$\end{corollary}

Assuming $\Theta$ is closed and convex, we can consider a naive projected subgradient minimization of $f$. Alternatively, if there exists a Bregman divergence $B(a,b)=\omega(b)-\omega(a)-\dotprod{\nabla\omega(a)}{b-a}$
for $a,b\in\Theta$ defined by a prox-function $\omega$, we can define the proximal mapping $P_{a}(b)=\argmin_{c\in\Theta} \left(\dotprod{b}{c-a}+B(a,c)\right)$ and consider accelerated gradient approaches \citep{Nesterov2005}. We summarize this idea in Algorithm~\ref{algo:discwass}.

\begin{algorithm}
	\begin{algorithmic}
		\caption{Wasserstein Barycenter in $P(X,\Theta)$\label{algo:discwass}} 
		\STATE \textbf{Inputs}: $X\in\Omega^n, \Theta\subset\Sigma_n$. For $i\leq N: Y_i\in\Omega^{m_i}, b_i\in\Sigma_{m_i}, p\in [1,\infty), t_0>0$.
		\STATE Form all $n\times m_i$ matrices $M_i=M_{XY_i}$, see Eq.~\eqref{eq:distMatrix}.
		\STATE Set $\hat{a}=\tilde{a}=\argmin_\Theta \omega$.
		\WHILE{not converged}
		\STATE $\beta=(t+1)/2$, $a\leftarrow (1-\beta^{-1}) \hat{a} + \beta^{-1}\tilde{a}$.
		\STATE Form subgradient $\pmb{\alpha}\leftarrow N^{-1}\sum_{i=1}^N \alpha_i^\star$ using all dual optima $\alpha_i^\star$ of $\bd(a,b_i,M_i)$.
		\STATE $\tilde{a}\leftarrow P_{a}(t_0\beta \pmb{\alpha})$.
		\STATE $\hat{a}\leftarrow (1-\beta^{-1}) \hat{a} + \beta^{-1}\tilde{a}, \,t\leftarrow t+1$.
		\ENDWHILE
	\end{algorithmic}
\end{algorithm}
Notice that when $\Theta=\Sigma_n$ and $B$ is the Kullback-Leibler divergence \citep{beck2003mirror}, we can initialize $\tilde{a}$ with $\one_n/n$ and use the multiplicative update to realize the proximal update: $\tilde{a}\leftarrow \tilde{a}\circ e^{-t_0\beta \pmb{\alpha}}; \tilde{a}\leftarrow \tilde{a}/\tilde{a}^T\one_n$, where $\circ$ is Schur's product. Alternative sets $\Theta$ for which  this projection can be easily carried out include, for instance, all (convex) level set of the entropy function $H$, namely $\Theta=\{a\in\Sigma_n | H(a)\geq \tau\}$ where $0\leq \tau\leq \log n$. 

%%%%%%%%%%%%%%%%%%%%%%%%%%%%%%%%%%%%%%%%%%%%%%%%%%%%%%%%
\subsection{Differentiability of $\bp(a,b,M_{XY})$ w.r.t $X$}\label{subsec:XEuclidean}
%%%%%%%%%%%%%%%%%%%%%%%%%%%%%%%%%%%%%%%%%%%%%%%%%%%%%%%%
We consider now the case where $\Omega=\mathbb{R}^d$ with $d\geq 1$, $D$ is the Euclidean distance and $p=2$. When $\Omega=\mathbb{R}^d$, a family of $n$ points $X$ and a family of $m$ points $Y$ can be  represented respectively as a matrix in $\mathbb{R}^{d\times n}$ and another in $\mathbb{R}^{d\times m}$. The pairwise squared-Euclidean distances between points in these sets can be recovered by writing $\bx\defeq\diag(X^TX)$ and $\by\defeq\diag(Y^TY)$, and observing that
$$
M_{XY}=\bx\one_m^T+\one_n\by^T - 2X^TY\in\mathbb{R}^{n\times m}.
$$
\textbf{Transport Cost as a function of $X$.}
Due to the margin constraints that apply if a matrix $T$ is in the polytope $U(a,b)$, we have:
$$
\begin{aligned}
	\dotprod{T}{M_{XY}} &=\dotprod{T}{\bx\one_d^T+\one_d^T\by - 2X^TY}\\&=\tr T^T\bx\one_d^T+\tr T^T\one_d^T\by - 2\dotprod{T}{X^TY}\\
	 &= \bx^Ta+\by^Tb - 2\dotprod{T}{X^TY}.
\end{aligned}
$$

Discarding constant terms in $\by$ and $b$, we have that minimizing $\bp(a,b,M_{XY})$ with respect to locations $X$ is equivalent to solving
\begin{equation}\label{eq:minmax}
\min_{\substack{X\in\mathbb{R}^{d\times k}}} \bx^Ta+\bp(a,b,-X^TY).
\end{equation}
As a function of $X$, that objective is the sum of a convex quadratic function of $X$ with a piecewise linear concave function, since 
$$\bp(a,b,-X^TY) = \min_{T\in U(a,b)}\dotprod{X}{-YT^T}$$
is the minimum of linear functions indexed by the vertices of the polytope $U(a,b)$. As a consequence, $\bp(a,b,M_{XY})$ is not convex with respect to $X$.\\

\textbf{Quadratic Approximation.} 
Suppose that $T^\star$ is optimal for problem $\bp(a,b,M_{XY})$. Updating Eq.~\eqref{eq:minmax},
\begin{multline*}\hskip-.4cm\bx^Ta- \dotprod{T^\star}{X^TY}=\norm{X\diag(a^{1/2})- YT^{\star T} \diag(a^{-1/2})}^2\\-\norm{YT^{\star T} \diag(a^{-1/2})}^2.\end{multline*}
Minimizing a local quadratic approximation of $\bp$ at $X$ yields thus the Newton update 
\begin{equation}\label{eq:newton}X\leftarrow YT^{\star T}\diag(a^{-1}).\end{equation} 
	A simple interpretation of this update is as follows: the matrix $T^{\star T}\diag(a^{-1})$ has $n$ column-vectors in the simplex $\Sigma_{m}$. The suggested update for $X$ is to replace it by $n$ barycenters of points enumerated in $Y$ with weights defined by the optimal transport $T^\star$. Note that, because the minimization problem we consider in $X$ is not convex to start with, one could be fairly creative when it comes to choosing $D$ and $p$ among other distances and exponents. This substitution would only involve more complicated gradients of $M_{XY}$ w.r.t. $X$ that would appear in Eq.~\eqref{eq:minmax}. 

%%%%%%%%%%%%%%%%%%%%%%%%%%%%%%%%%%%%%%%%%%%%%%%%%%%%%%%%
\subsection{Free Support: Minimizing $f$ over $P_k(\mathbb{R}^d,\Theta)$}\label{subsec:XEuclidean}
%%%%%%%%%%%%%%%%%%%%%%%%%%%%%%%%%%%%%%%%%%%%%%%%%%%%%%%%
We now consider, as a natural extension of \S\ref{subsec:Xrestricted} when $\Omega=\mathbb{R}^d$, the problem of minimizing $f$ over a probability measure $\mu$ that is (i) supported by \emph{at most $k$ atoms} described in $X$, a matrix of size $d\times k$, (ii) with weights in $a\in\Theta\subset\Sigma_k$.

\textbf{Alternating Optimization.} To obtain an approximate minimizer of $f(a,X)$ we propose in Algorithm~\ref{algo:general} to update alternatively locations $X$ (with the Newton step defined in Eq.~\ref{eq:newton}) and weights $a$ (with Algorithm~\ref{algo:discwass}). 

\begin{algorithm}
	\begin{algorithmic}
		\caption{$2$-Wasserstein Barycenter in $P_k(\mathbb{R}^d,\Theta)$\label{algo:general}}
		\STATE \textbf{Input}: $Y_i\in\mathbb{R}^{d\times m_i}, b_i\in\Sigma_{m_i}$ for $i\leq N.$
		\STATE initialize $X\in\mathbb{R}^{d\times k}$ and $a\in\Theta$
		\WHILE{$X$ and $a$ have not converged}
				\STATE $a\leftarrow a^\star$ using Algorithm~\ref{algo:discwass}.
				\FOR {$i\in (1,\dots,N)$} 
				\STATE $T_i^\star \leftarrow$ optimal solution of $\bp(a,b_i;M_{XY_i})$
				\ENDFOR
				\STATE $X \leftarrow (1-\theta)X + \theta \left(\frac{1}{N}\sum_{i=1}^N Y_iT_i^{\star T}\right)\diag(a^{-1})$, setting $\theta\in[0,1]$ with line-search or a preset value.
		\ENDWHILE
	\end{algorithmic}
\end{algorithm}

\textbf{Algorithm~\ref{algo:general} and \citeauthor{lloyd1982least}/\citeauthor{ng2000note} Algorithms.} As mentioned in \S\ref{sec:back}, minimizing $f$ defined in Eq.~\eqref{eq:f} over $P_k(\mathbb{R}^d)$, with $N=1$, $p=2$ and no constraints on the weights ($\Theta=\Sigma_k$), is equivalent to solving the $k$-means problem applied to the set of points enumerated in $\nu_1$. In that particular case, Algorithm~\ref{algo:general} is also equivalent to \citeauthor{lloyd1982least}'s algorithm. Indeed, the assignment of the weight of each point to its closest centroid in \citeauthor{lloyd1982least}'s algorithm (the maximization step) is equivalent to the computation of $a^\star$ in ours, whereas the re-centering step (the expectation step) is equivalent to our update for $X$ using the optimal transport, which is in that case the trivial transport that assigns the weight (divided by $N$) of each atom in $Y_i$ to its closest neighbor in $X$. When the weight vector $a$ is constrained to be uniform ($\Theta=\{\one_k/k\}$),
\citet{ng2000note} proposed a heuristic to obtain uniform $k$-means that is also equivalent to Algorithm~\ref{algo:general}, and which also relies on the repeated computation of optimal transports. For more general sets $\Theta$, Algorithm~\ref{algo:discwass} ensures that the weights $a$ remain in $\Theta$ at each iteration of Algorithm~\ref{algo:general}, which cannot be guaranteed by neither \citeauthor{lloyd1982least}'s nor \citeauthor{ng2000note}'s approach. 

\textbf{Algorithm~\ref{algo:general} and \citeauthor{reich2012}'s \citeyearpar{reich2012} Transform.} \citet{reich2012} has recently suggested to approximate a weighted measure $\nu$ by a uniform measure supported on as many atoms. This approximation is motivated by optimal transport theory, notably asymptotic results by \citet{mccann1995existence}, but does not attempt to minimize, as we do in Algorithm~\ref{algo:general}, any Wasserstein distance between that approximation and the original measure. This approach results in one application of the Newton update defined in Eq.~\eqref{eq:newton}, when $X$ is first initialized to $Y$ and $a=\one_m/m$ to compute the optimal transport $T^\star$.

\textbf{Summary} We have proposed two original algorithms to compute Wasserstein barycenters of probability measures: one which applies when the support of the barycenter is fixed and its weights are constrained to lie in a convex subset $\Theta$ of the simplex, another which can be used when the support can be chosen freely. These algorithms are relatively simple, yet---to the best of our knowledge---novel. We suspect these approaches were not considered before because of their prohibitive computational cost: Algorithm~\ref{algo:discwass} computes at each iteration the dual optima of $N$ transportation problems to form a subgradient, each with $n+m_i$ variables and $n\times m_i$ inequality constraints. Algorithm~\ref{algo:general} incurs an even higher cost, since it involves running Algorithm~\ref{algo:discwass} at each iteration, in addition to solving $N$ primal optimal transport problems to form a subgradient to update $X$. Since both objectives rely on subgradient descent schemes, they are also likely to suffer from a very slow convergence. We propose to solve these issues by following \citeauthor{cuturi2013sinkhorn}'s approach \citeyearpar{cuturi2013sinkhorn} to smooth the objective $f$ and obtain strictly convex objectives whose gradients can be computed more efficiently.

%%%%%%%%%%%%%%%%%%%%%%%%%%%%%%%%%%%%%%%%%%%%%%%%%%%%%%%%
\section{Smoothed Dual and Primal Problems}\label{sec:smooth}
%%%%%%%%%%%%%%%%%%%%%%%%%%%%%%%%%%%%%%%%%%%%%%%%%%%%%%%%
To circumvent the major computational roadblock posed by the repeated computation of primal and dual optimal transports, we extend \citeauthor{cuturi2013sinkhorn}'s approach \citeyearpar{cuturi2013sinkhorn} to obtain smooth and strictly convex approximations of both primal and dual problems $\bp$ and $\bd$. The matrix scaling approach advocated by \citeauthor{cuturi2013sinkhorn} was motivated by the fact that it provided a fast approximation $\bp_\lambda$ to $\bp$. We show here that the same approach can be used to smooth the objective $f$ and recover for a cheap computational price its gradients w.r.t. $a$ and $X$.

\subsection{Regularized Primal and Smoothed Dual}
A $n\times m$ transport $T$, which is by definition in the $nm$-simplex, has entropy
$h(T)\defeq -\sum_{i,j=1}^{n,m} t_{ij}\log(t_{ij}).$
\citet{cuturi2013sinkhorn} has recently proposed to consider, for $\lambda>0$, a regularized primal transport problem $\bp_\lambda$ as
$$
\bp_\lambda(a,b;M) = \min_{T\in U(a,b)} \dotprod{X}{M} -\frac{1}{\lambda}h(T).
$$

We introduce in this work its dual problem, which is a smoothed version of the original dual transportation problem, where the positivity constraints of each term $m_{ij}-\alpha_i-\beta_j$ have been replaced by penalties $\frac{1}{\lambda}e^{-\lambda \left(m_{ij}-\alpha_i-\beta_j\right)}$:
$$
\bd_\lambda(a,b;M) =\!\!\!\!\!\! \max_{(\alpha,\beta)\in \mathbb{R}^{n+m}} \!\!\!\!\!\!\alpha^Ta+\beta^Tb -\!\!\!\!\!\!\sum_{i\leq n,j \leq m}\!\!\!\!\!\!\frac{e^{-\lambda \left(m_{ij}-\alpha_i-\beta_j\right)}}{\lambda}.
$$

These two problems are related below in the sense that their respective optimal solutions are linked by a unique positive vector $u\in\mathbb{R}_+^{n}$:
\begin{proposition}\label{prop:primdual}Let $K$ be the elementwise exponential of $-\lambda M_{XY}$, $K\defeq e^{-\lambda M_{XY}}$. Then there exists a pair of vectors $(u,v)\in\mathbb{R}_+^{n}\times\mathbb{R}_+^{m}$ such that the optimal solutions of $\bp_\lambda$ and $\bd_\lambda$ are respectively given by
	$$T_\lambda^\star = \diag(u)K\diag(v), \alpha_\lambda^\star = -\frac{\log(u)}{\lambda}+\frac{\log(u)^T\one_n}{\lambda n}\one_n.$$
\end{proposition}
\begin{proof} The result follows from the Lagrange method of multipliers for the primal as shown by \citet[Lemma 2]{cuturi2013sinkhorn}, and a direct application of first-order conditions for the dual, which is an unconstrained convex problem. The term $\frac{\log(u)^T\one_n}{\lambda n}\one_n$ in the definition of $\alpha_\lambda^\star$ is used to normalize $\alpha_\lambda^\star$ so that it sums to zero as discussed in the end of \S\ref{subsec:dualityconvexity}.\end{proof}

\subsection{Matrix Scaling Computation of $(u,v)$}

The positive vectors $(u,v)$ mentioned in Proposition~\ref{prop:primdual} can be computed through \citeauthor{sinkhorn1967diagonal}'s matrix scaling algorithm applied to $K$, as outlined in Algorithm~\ref{algo:sk}:
% \begin{lemma}[\citeauthor{sinkhorn1967diagonal}, \citeyear{sinkhorn1967diagonal}]\label{theo:SK} For any positive $n\times m$ matrix $A$ and positive vectors $a\in\Sigma_n$ and $b\in\Sigma_m$, there exists $(u,v)$ in $\mathbb{R}^n_+\times \mathbb{R}^m_+$, unique up to scalar multiplication, such that $\diag(u)A\diag(v)\in U(a,b)$. The pair $(u,v)$ can be recovered with any fixed point of the map
% 	 $$(u,v) \mapsto (Av^{-1}./b,A^Tu^{-1}./a).$$
% \end{lemma}

	\begin{lemma}[\citeauthor{sinkhorn1967diagonal}, \citeyear{sinkhorn1967diagonal}]\label{theo:SK} For any positive matrix $A$ in $\mathbb{R}_+^{n\times m}$ and positive probability vectors $a\in\Sigma_n$ and $b\in\Sigma_m$, there exist positive vectors $u\in \mathbb{R}^n_+$ and $v\in\mathbb{R}^m_+$, unique up to scalar multiplication, such that $\diag(u)A\diag(v)\in U(a,b)$. Such a pair $(u,v)$ can be recovered as a fixed point of the Sinkhorn map
		 $$(u,v) \mapsto (Av^{-1}./b,A^Tu^{-1}./a).$$
	\end{lemma}

The convergence of the algorithm is linear when using Hilbert's projective metric between the scaling factors \citep[\S3]{franklin1989scaling}. Although we use this algorithm in our experiments because of its simplicity, other algorithms exist \citep{knight2012fast} which are known to be more reliable numerically when $\lambda$ is large.

\begin{algorithm}[t]
        \begin{algorithmic}
	\caption{Smoothed Primal $T^\star_\lambda$ and Dual $\alpha_\lambda^\star$ Optima}\label{algo:sk}
          \STATE \textbf{Input} $M,\lambda,a,b$
		  \STATE $K=\exp(-\lambda M)$; 
		  \STATE $\widetilde{K}=\diag(a^{-1}) K$ \% use \texttt{bsxfun(@rdivide,$K$,$a$)}
		  \STATE Set $u=$\texttt{ones($n$,1)/$n$};
		  \WHILE{$u$ changes}
		  \STATE \texttt{$u=1./(\widetilde{K}(b./(K^Tu)))$}.
		  \ENDWHILE
		  \STATE $v=b./(K^Tu).$
		  \STATE  $\alpha^\star_\lambda=-\frac{1}{\lambda}\log(u)+\frac{\log(u)^T\one_n}{\lambda n}\one_n$. 
		  \STATE  $T^{\star}_\lambda=\diag(u)K\diag(v)$.
     	  \STATE \% use $\texttt{bsxfun}(\texttt{@times},v,(\texttt{bsxfun}(\texttt{@times},K,u))')$;
        \end{algorithmic}
     \end{algorithm}

\begin{figure*}[t]
  \includegraphics[width=\textwidth,height=5.5cm]{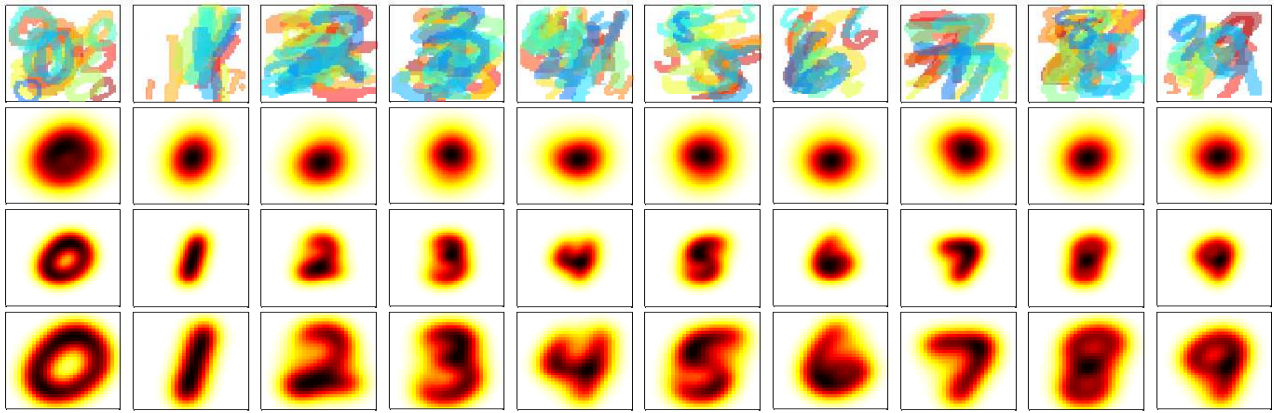}
  \caption{(top) For each digit, $15$ out of the $\approx 5.000$ scaled and translated images considered for each barycenter. (bottom) Barycenters after $t=1,10,60$ gradient steps. For $t=60$, images are cropped to show the $30\times30$ central pixels.}\label{fig:wass}
\end{figure*}

\textbf{Summary:} Given a smoothing parameter $\lambda>0$, using Sinkhorn's algorithm on matrix $K$, defined as the elementwise exponential of $-\lambda M$ (the pairwise Gaussian kernel matrix between the supports $X$ and $Y$ when $p=2$, using bandwidth $\sigma=1/\sqrt{2\lambda}$) we can recover smoothed optima $\alpha^\star_\lambda$ and $T^\star_\lambda$ for \emph{both} smoothed primal $\bp_\lambda$ and dual $\bd_\lambda$ transport problems. To take advantage of this, we simply propose to substitute the smoothed optima $\alpha^\star_\lambda$ and $T^\star_\lambda$ to the original optima $\alpha^\star$ and $T^\star$ that appear in Algorithms~\ref{algo:discwass} and \ref{algo:general}.

%%%%%%%%%%%%%%%%%%%%%%%%%%%%%%%%%%%%%%%%
\section{Applications}\label{sec:exp}
%%%%%%%%%%%%%%%%%%%%%%%%%%%%%%%%%%%%%%%%
We present two applications, one of Algorithm~\ref{algo:discwass} and one of Algorithm \ref{algo:general}, that both rely on the smooth approximations presented in \S\ref{sec:smooth}. The settings we consider involve computing respectively tens of thousands or tens of high-dimensional optimal transport problems---2.500$\times$2.500 for the first application, $57.647\times 48$ for the second---which cannot be realistically carried out using network flow solvers. Using network flow solvers, the resolution of a single transport problem of these dimensions could take between several minutes to several hours. We also take advantage in the first application of the fact that Algorithm~\ref{algo:sk} can be run efficiently on GPGPUs using vectorized code \citep[Alg.1]{cuturi2013sinkhorn}.

\subsection{Visualization of Perturbed Images}
We use $50.000$ images of the MNIST database, with approximately $5.000$ images for each digit from 0 to 9. Each image (originally $20\times 20$ pixels) is scaled randomly, uniformly between half-size and double-size, and translated randomly within a $50\times 50$ grid, with a bias towards corners. We display intermediate barycenter solutions for each of these 10 datasets of images for $t=1,10,60$ gradient iterations. $\lambda$ is set to $60/\text{median}(M)$, where $M$ is the squared-Euclidean distance matrix between all 2,500 pixels in the grid. Using a Quadro K5000 GPU with close to 1500 cores, the computation of a single barycenter takes about 2 hours to reach 100 iterations. Because we use warm starts to initialize $u$ in Algorithm~\ref{algo:sk} at each iteration of Algorithm~\ref{algo:discwass}, the first iterations are typically more computationally intensive than those carried out near the end.

\begin{figure}[h!]
  \hskip-1.2cm\includegraphics[width=11cm]{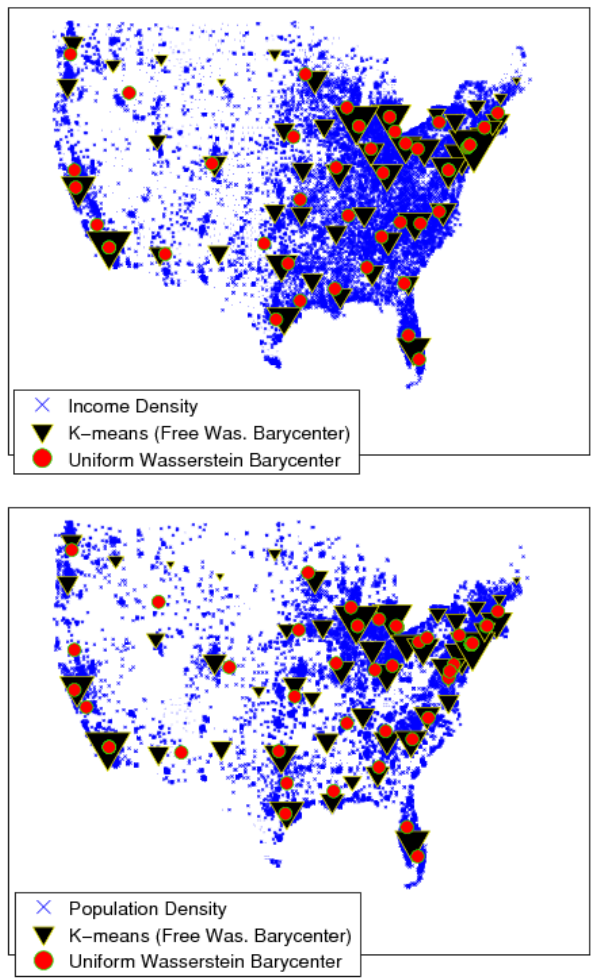}
  \caption{Comparison of two Wasserstein barycenters on spatial repartitions of income and population in the 48 contiguous states, using as many ($k=48$) centroids. The size of each of the 57.647 blue crosses is proportional to the local average of the relevant variable (income above and population below) at that location, normalized to sum to 1. Each downward triangle is a centroid of the $k$-means clustering (equivalent to a Wasserstein barycenter with $\Theta=\Sigma_k$) whose size is  proportional to the portion of mass captured by that centroid. Red dots indicate centroids obtained with a uniform constraint on the weights, $\Theta=\{\one_k/k\}$. Since such centroids are constrained to carry a fixed portion of the total weight, one can observe that they provide a more balanced clustering than the $k$-means solution.}\label{fig:america}
\end{figure}

\subsection{Clustering with Uniform Centroids}
In practice, the $k$-means cost function applied to a given empirical measure could be minimized with a set of centroids $X$ and weight vector $a$ such that the entropy of $a$ is very small. This can occur when most of the original points in the dataset are attributed to a very small subset of the $k$ centroids, and could be undesirable in applications of $k$-means where a more regular attribution is sought. For instance, in sensor deployment, when each centroid (sensor) is limited in the number of data points (users) it can serve, we would like to ensure that the attributions agree with those limits.

Whereas the original $k$-means cannot take into account such limits, we can ensure them using Algorithm~\ref{algo:general}.  We illustrate the difference between looking for optimal centroids with ``free'' assignments ($\Theta=\Sigma_k$), and looking for optimal ``uniform'' centroids with constrained assignments ($\Theta=\{\one_k/k\}$) using US census data for income and population repartitions across 57.647 spatial locations in the 48 contiguous states. These weighted points can be interpreted as two empirical measures on $\mathbb{R}^{2}$ with weights directly proportional to these respective quantities. We initialize both ``free'' and ``uniform'' clustering with the actual 48 state capitals. Results displayed in Figure~\ref{fig:america} show that by forcing our approximation to be uniform, we recover centroids that induce a more balanced clustering. Indeed, each cell of the Voronoi diagram built with these centroids is now constrained to hold the same aggregate wealth or population. These centroids could form the new state capitals of equally rich or equally populated states. On an algorithmic note, we notice in Figure~\ref{fig:graphe} that Algorithm~\ref{algo:general} converges to its (local) optimum at a speed which is directly comparable to that of the $k$-means in terms of iterations, with a relatively modest computational overhead. Unsurprisingly, the Wasserstein distance between the clusters and the original measure is higher when adding uniform constraints on the weights.

\begin{figure}[h!]
  \centering\includegraphics[width=7.5cm]{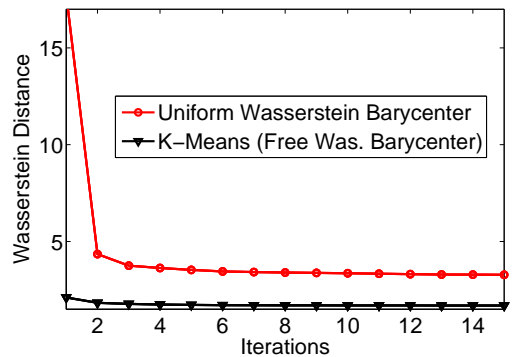}
  \caption{Wasserstein distance of the uniform Wasserstein barycenter (weights constrained to be in $\Theta=\{\one_k/k\}$) and its unconstrained equivalent ($k$-means) to the income empirical measure. Note that, because of the constraints on weights, the Wasserstein distance of the uniform Wasserstein barycenter is necessarily larger. On a single CPU core, these computations require 12.5 seconds for the constrained case, using Sinkhorn's approximation, and 1.55 seconds for the regular $k$-means algorithm. Using a regular transportation solver, computing the optimal transport from the 57.647 points to the 48 centroids would require about 1 hour for a single iteration}\label{fig:graphe}
\end{figure}

\paragraph{Conclusion} We have proposed in this paper two original algorithms to compute Wasserstein barycenters of empirical measures. Using these algorithms in practice for measures of large support is a daunting task for two reasons: they are inherently slow because they rely on the subgradient method; the computation of these subgradients involves solving optimal and dual optimal transport problems. Both issues can be substantially alleviated by smoothing the primal optimal transport problem with an entropic penalty and considering its dual. Both smoothed problems admit gradients which can be computed efficiently using only matrix vector products. Our aim in proposing such algorithms is to demonstrate that Wasserstein barycenters can be used for visualization, constrained clustering, and hopefully as a core component within more complex data analysis techniques in future applications. We also believe that our smoothing approach can be directly applied to more complex variational problems that involve multiple Wasserstein distances, such as Wasserstein propagation~\citep{solomon2014wasserstein}.

\paragraph{Acknowledgements}We thank reviewers for their comments and Gabriel Peyr\'e for fruitful discussions. MC was supported by grant 26700002 from JSPS. AD was partially supported by EPSRC.

 {\small{
\bibliographystyle{abbrvnat}

}}

\end{document}